\documentclass[runningheads]{llncs}
\usepackage{amsmath,amssymb,mathtools}
\usepackage{graphicx,subcaption,booktabs}
\usepackage{algorithm,algorithmic}
\usepackage{float}

\begin{document}

\title{Learning with Category-Equivariant Representations for Human Activity Recognition}
\titlerunning{Learning with Category-Equivariant Representations}
\author{Yoshihiro Maruyama\inst{1,2}}
\authorrunning{Y. Maruyama}
\institute{School of Informatics, Nagoya University, Japan\\
\email{maruyama@i.nagoya-u.ac.jp}
\and
School of Computing, Australian National University, Australia\\
\email{yoshihiro.maruyama@anu.edu.au}
}

\maketitle

\begin{abstract}
Human activity recognition is challenging because sensor signals shift with context, motion, and environment; effective models must therefore remain stable as the world around them changes. We introduce a categorical symmetry-aware learning framework that captures how signals vary over time, scale, and sensor hierarchy. We build these factors into the structure of feature representations, yielding models that automatically preserve the relationships between sensors and remain stable under realistic distortions such as time shifts, amplitude drift, and device orientation changes. On the UCI Human Activity Recognition benchmark, this categorical symmetry-driven design improves out-of-distribution accuracy by approx.\ 46 percentage points (approx.\ 3.6× over the baseline), demonstrating that abstract symmetry principles can translate into concrete performance gains in everyday sensing tasks via category-equivariant representation theory.
\end{abstract}

\section{Introduction}
\label{sec:intro}

Recognizing human activity from sensor data demands models that remain dependable even as environmental conditions and sensor orientations change. In practice, raw inertial signals undergo temporal misalignment, amplitude drift, and device pose variation; these shifts routinely degrade standard representations trained under fixed conditions. Classical signal processing has long emphasized invariances to such effects (e.g., shift properties of Fourier magnitudes and normalization for gain) and elastic matching for timing variability \cite{OppenheimSchafer2009,SakoeChiba1978}. Yet in contemporary learning systems, robustness is often pursued indirectly via data augmentation or distributional objectives \cite{Arjovsky2019}, which may provide limited guarantees when the symmetry structure is rich and multi-faceted, as in Human Activity Recognition (HAR) benchmarks \cite{Anguita2013,Kwapisz2011}.

In this paper we develop a principled method to address robustness through \emph{equivariance}: designing representations that commute with symmetry actions. Group-equivariant convolution and its steerable and gauge-equivariant extensions have delivered strong invariance and parameter sharing on Euclidean and manifold domains \cite{CohenWelling2016,CohenWelling2017,WeilerCesa2019,Cohen2019Gauge,KondorTrivedi2018,Finzi2021EMLP}; related ideas underlie permutation-invariant and equivariant models for sets and graphs \cite{Zaheer2017,Maron2019,KerivenPeyre2019}. These developments are surveyed within the broader field of geometric deep learning \cite{Bronstein2021}. However, many real sensing pipelines exhibit symmetry beyond groups: they combine \emph{compositional} transformations (e.g., time shifts and gain rescalings) with \emph{hierarchical} relations (e.g., axes $\rightarrow$ magnitude $\rightarrow$ total), the latter being naturally described by partial orders (posets). Such settings call for a language that unifies actions and hierarchies.

Category theory \cite{MacLane1998,Spivak2014,FongSpivak2019}  provides that unifying language: symmetries and data domains can be organized as \emph{categories}, and equivariant models become \emph{natural transformations} between functors. Building on this idea, we propose a \emph{category-equivariant} framework for HAR in which the symmetry category is a product of a group and a poset. Concretely, we take a group that captures temporal shifts and per-sensor gain rescalings, and a poset that encodes the sensor hierarchy from tri-axial blocks to magnitudes and a total aggregate. Representations are constructed so that pushing signals along these morphisms and then representing is equal to representing first and then transporting features, which amounts to the naturality condition in terms of category theory.

Our instantiation is deliberately simple and transparent: per-sensor RMS normalization (gain), axis-to-magnitude pooling (hierarchy), and low-frequency Fourier magnitudes (time-shift) yield a compact feature vector that is invariant/equivariant to the group and natural with respect to the poset. A lightweight linear classifier suffices on top. Despite its simplicity, this category-aware representation exhibits substantially improved robustness under test-time shifts in timing, amplitude, and orientation on UCI HAR \cite{Anguita2013}, while maintaining clarity about why it works. Our approach complements, rather than competes with, group-based architectures \cite{CohenWelling2016,CohenWelling2017,WeilerCesa2019,Cohen2019Gauge,KondorTrivedi2018,Finzi2021EMLP}, set/graph equivariance \cite{Zaheer2017,Maron2019,KerivenPeyre2019}, and distributional robustness \cite{Arjovsky2019}: our approach targets a different structural axis, namely the joint presence of compositional actions and hierarchical relations.

The contributions of this paper can be summarized as follows.
\begin{enumerate}
\item We formulate a \emph{Group$\times$Poset} symmetry category for inertial sensing and define category-equivariant representations as natural transformations between data and feature functors.  
\item We give a constructive realization for HAR: RMS normalization (gain), axis-to-magnitude pooling (hierarchy), and Fourier magnitudes (time-shift), yielding a compact, interpretable feature map grounded in classical signal principles \cite{OppenheimSchafer2009}.  
\item We demonstrate substantial out-of-distribution robustness on UCI HAR \cite{Anguita2013} under time/gain/pose perturbations and provide ablations clarifying the separate and combined effects of the group and poset components.  
\item We give a theoretical analysis: naturality on generators implies full category-equivariance, and invariant blocks enjoy exact risk invariance under group perturbations; non-equivariant baselines degrade in proportion to their expected orbit diameter.
\item Taken together, these results suggest that elevating equivariance from groups to categories—where actions and hierarchies coexist—offers a compact and effective route to robustness in everyday sensing tasks, complementing existing geometric and distributional approaches \cite{Bronstein2021,Arjovsky2019}.
\end{enumerate}

The rest of the paper is organized as follows. Section~\ref{sec:framework} introduces the category-equivariant framework, defining the product category of a group and a poset and detailing its instantiation on Human Activity Recognition. Section~\ref{sec:experiments} presents the experimental protocol and results, including ablations and robustness analysis under out-of-distribution conditions. Section~\ref{sec:theory} develops the theoretical foundations, proving category-equivariance and deriving robustness bounds. The paper concludes with a brief discussion of implications and future directions.

%Category theory provides a natural language to unify symmetries as functors and equivariant models as natural transformations. This perspective subsumes group-equivariant, group-equivariant, and sheaf- or order-equivariant networks.

%(1) We formalize a Group×Poset category of symmetries; (2) we construct explicit equivariant features for HAR combining time–gain (group) and sensor-hierarchy (poset) invariances; (3) we show large robustness gains under OOD perturbations; (4) we prove that our feature map is a natural transformation guaranteeing categorical commutativity.

\section{Categorical Symmetry and Category-Equivariant Representations}
\label{sec:framework}

\paragraph{Problem setting.}
We study supervised learning with inputs that possess \emph{compositional} temporal/amplitude structure and \emph{hierarchical} sensor structure.  
Let $T\in\mathbb{N}$ be a fixed window length and let $S$ be a finite set of sensors (e.g.\ $S=\{\mathrm{ACC},\mathrm{GYRO}\}$).\footnote{ACC = the tri‑axial accelerometer signals. GYRO = tri‑axial gyroscope signals.}
Each sensor $s\in S$ provides a tri-axial signal $x_s\in(\mathbb{R}^{T})^3$.  
The learning objective is to construct a representation that is (i) equivariant/invariant to temporal shifts and per-sensor gain rescalings, and (ii) \emph{natural} with respect to a sensor hierarchy (axes $\to$ magnitude $\to$ total).

\subsection{Product category of symmetries}

\paragraph{Group $M$ (time and gain).}
Let $C_T$ denote the cyclic group of length-$T$ circular time-shifts acting on $\mathbb{R}^T$ by $(\tau^t x)(n)=x(n+t)$ (indices modulo $T$).  
Let $\Lambda=(\mathbb{R}_{>0})^{S}$ be the commutative group of per-sensor gains with componentwise multiplication; for $\lambda\in\Lambda$ and a tri-axial block $x_s\in(\mathbb{R}^T)^3$ we write $(\lambda\!\cdot\!x)_s=\lambda_s\,x_s$.  
Define the (commutative) group
\[
M \;\coloneqq\; C_T \times \Lambda,
\qquad (t,\lambda)\cdot(t',\lambda')=(t+t',\,\lambda\lambda').
\]

\paragraph{Poset $P$ (sensor hierarchy).}
For each $s\in S$ introduce two nodes $s{:}\mathrm{axes}$ and $s{:}\mathrm{mag}$ with order $s{:}\mathrm{axes}\preceq s{:}\mathrm{mag}$, and a global node $\mathrm{TOTAL}$ with $s{:}\mathrm{mag}\preceq \mathrm{TOTAL}$.  
Thus $P$ is the thin category whose objects are
\[
\mathrm{Ob}(P)=\{\,s{:}\mathrm{axes},\,s{:}\mathrm{mag}\mid s\in S\,\}\cup\{\mathrm{TOTAL}\},
\]
and whose non-identity morphisms are precisely the inclusions described above.\footnote{$\mathrm{Ob}(\mathcal{C})$ denotes the collection of objects of $\mathcal{C}$, and $\mathrm{Mor}(\mathcal{C})$ the collection of morphisms of $\mathcal{C}$.}

\paragraph{Product category.}
Let $B M$ be the one-object category with endomorphisms $M$.
We work over $\mathcal{C}=B M\times P$, so $\mathrm{Ob}(\mathcal{C})=\mathrm{Ob}(P)$.  
A morphism is a pair
\[
(m,u)\colon a\longrightarrow b \quad\text{with}\quad m\in M,\; u\colon a\to b\ \text{in }P,
\]
composed by $(m',u')\circ(m,u)=(m'm,\,u'u)$.

\subsection{Data and model functors; category-equivariant representations}

We view signal formation and representation as covariant functors into $\mathbf{Meas}$, the category of standard Borel spaces and measurable maps.%
\footnote{If $\mathbf{Meas}$ seems too broad, one may instead work in the full subcategory
of Euclidean spaces equipped with their standard Borel $\sigma$-algebras and
measurable maps. We work on measurable maps since we use non-continuous measurable maps below.}

\begin{definition}[Data and model functors]
A \emph{data functor} is a covariant functor $X\colon \mathcal{C}\to\mathbf{Meas}$ assigning to each $a\in\mathrm{Ob}(P)$ a standard Borel space $X(a)$ of signals and to each $(m,u)$ a measurable map $X(m,u)$.\footnote{$X(a)$ is usually a Euclidean space with the standard Borel $\sigma$-algebra.}  
A \emph{model functor} is a covariant functor $Y\colon \mathcal{C}\to\mathbf{Meas}$ specifying the target feature spaces and their induced transport along $\mathcal{C}$.
\end{definition}

Concretely, for each $s\in S$,
\[
X(s{:}\mathrm{axes})=(\mathbb{R}^T)^3,\qquad
X(s{:}\mathrm{mag})=\mathbb{R}^T,\qquad
X(\mathrm{TOTAL})=\prod_{r\in S}\mathbb{R}^T.
\]
and for the non-identity $u$ in $P$ we set
\[
X(u)=
\begin{cases}
m_s:(\mathbb{R}^T)^3\to\mathbb{R}^T, &
m_s(x,y,z)=\sqrt{x^2+y^2+z^2}\text{ (pointwise)},\\[3pt]
\iota_s:\mathbb{R}^T\hookrightarrow\prod_{r\in S}\mathbb{R}^T, &
\text{embedding into the $s$-th coordinate.}
\end{cases}
\]
For $(t,\lambda)\in C_T\times\Lambda$, the action is $X(t,\lambda)(x)=\lambda\cdot(\tau_t x)$ on the appropriate signal spaces, with $(\tau_t x)(n)=x(n+t)$ (indices modulo $T$).

For completeness, the functorial action on a general morphism
$(m,u):a\to b$ in $C=\mathbf{B}M\times P$ is given by
\[
X(m,u)\;=\;X(u)\circ X(m,\mathrm{id}),
\qquad
Y(m,u)\;=\;Y(u)\circ Y(m,\mathrm{id}),
\]
where $X(m,\mathrm{id})$ and $Y(m,\mathrm{id})$ denote the $M$-actions
(time-shift and gain) on the corresponding objects of $P$.
This composition rule ensures compatibility with
the product-category structure (see also the theoretical analysis section below).

\begin{definition}[Category-equivariant representation]
A representation is a family of (possibly nonlinear) maps $\Phi=\{\Phi_a\}_{a\in\mathrm{Ob}(P)}$ with $\Phi_a\colon X(a)\to Y(a)$.  
It is \emph{category-equivariant} if for all morphisms $(m,u)\colon a\to b$ in $\mathcal{C}$ the following naturality condition holds:
\begin{equation}
\label{eq:naturality}
Y(m,u)\circ \Phi_a \;=\; \Phi_b \circ X(m,u).
\end{equation}
\end{definition}

Equivariance amounts to naturality, not just in the present setting, but also in the general setting of category theory.

\subsection{Concrete realization on HAR (Group × Poset)}

\paragraph{Spaces.}
For each sensor $s\in S=\{\mathrm{ACC},\mathrm{GYRO}\}$ we set
\[
X(s{:}\mathrm{axes})=(\mathbb{R}^T)^3,\qquad
X(s{:}\mathrm{mag})=\mathbb{R}^T,
\]
with $T=128$ in our experiments.
The total node aggregates per-sensor magnitudes:
\[
X(\mathrm{TOTAL})=\prod_{r\in S}\mathbb{R}^T.
\]
For feature spaces we let
\[
Y(s{:}\mathrm{axes})=\mathbb{R}^{k}\times\mathbb{R},\qquad
Y(s{:}\mathrm{mag})=\mathbb{R}^{k},\qquad
Y(\mathrm{TOTAL})=\mathbb{R}^{k},
\]
where the additional $\mathbb{R}$ factor in $Y(s{:}\mathrm{axes})$
stores the amplitude scalar $a_s=\|x_s\|_2$.

\paragraph{Group actions on $X$ and $Y$.}
On inputs, $(t,\lambda)\in M=C_T\times\Lambda$ acts by circular shift and
per-sensor gain:
\[
X(t,\lambda)(x_s)(n)=\lambda_s\,x_s(n+t),
\]
where indices are taken modulo $T$.
On features, the $C_T$ component acts trivially (since spectral magnitudes
are shift-invariant), while the $\Lambda$ component acts trivially on spectral
blocks and multiplicatively on the amplitude scalars:
\[
Y(t,\lambda)(a_s)=\lambda_s\,a_s.
\]

\paragraph{Representation maps.}
For each tri-axial window $x_s=(x_s^{(1)},x_s^{(2)},x_s^{(3)})\in X(s{:}\mathrm{axes})$, define:
\begin{enumerate}
\item \textbf{Gain normalization} (group invariance):
\[
\widehat{x}_s \;\coloneqq\; \frac{x_s}{\|x_s\|_2},
\qquad 
\|x_s\|_2^2 \;=\; \sum_{i=1}^3\sum_{n=1}^T\!\big(x_s^{(i)}(n)\big)^2,\qquad (\|x_s\|_2>0).
\]
If $\|x_s\|_2=0$, we set $\widehat{x}_s=0$.
\item \textbf{Axis-to-magnitude pooling} (poset arrow $s{:}\mathrm{axes}\to s{:}\mathrm{mag}$):
\[
\tilde m_s(n)\;\coloneqq\;\sqrt{\sum_{i=1}^3\big(\widehat{x}_s^{(i)}(n)\big)^2}\;\in\;\mathbb{R}^T.
\]
Then this map is invariant to instantaneous rotations $R\in SO(3)$, since 
$\|R\,\widehat{x}_s(n)\|_2=\|\widehat{x}_s(n)\|_2$, 
and it satisfies the commutation identity
\[
\tilde m_s(x/\|x\|_2)\;=\;\frac{\tilde m_s(x)}{\|\tilde m_s(x)\|_2},
\]
ensuring compatibility between normalization and pooling.
\item \textbf{Time-shift invariance} (group invariance):  
For $z\in X(s{:}\mathrm{mag})$, let $N(z)=z/\|z\|_2$ if $\|z\|_2>0$ and $N(0)=0$.  
Define
\[
\Phi_{s{:}\mathrm{mag}}(z)
\;\coloneqq\;
\big|\mathrm{rFFT}\!\big(N(z)\big)\big|_{[1..k]}\in\mathbb{R}^{k},
\]
where $\mathrm{rFFT}$ denotes the real-signal half-spectrum and $[1..k]$
are the lowest non-DC frequency bins.
The discrete real FFT satisfies, for each frequency bin $r$,
\[
\widehat{\tau_t z}[r]
= e^{-2\pi i r t / T}\,\widehat{z}[r],
\]
so every bin is multiplied by a unit-modulus phase factor.
Taking absolute values removes this phase, yielding
$|\widehat{\tau_t z}[r]|=|\widehat{z}[r]|$ for all $r$;
hence the spectral magnitudes are invariant to circular shifts.
\item \textbf{Axis-level map} (domain coherence):
\[
\Phi_{s{:}\mathrm{axes}}(x_s)\;\coloneqq\;
\big(\,\Phi_{s{:}\mathrm{mag}}\!\big(\tilde m_s(x_s)\big),\;
      a_s(x_s)\,\big)
   \in \mathbb{R}^k\times\mathbb{R},
\]
ensuring that
\[
Y(u_s)\!\circ\!\Phi_{s{:}\mathrm{axes}}
=\Phi_{s{:}\mathrm{mag}}\!\circ\!X(u_s),
\qquad
u_s:s{:}\mathrm{axes}\to s{:}\mathrm{mag},
\]
where $Y(u_s)=\mathrm{pr}_1$ projects
$\mathbb{R}^k\times\mathbb{R}\to\mathbb{R}^k$.
\item \textbf{TOTAL node} (poset arrows $s{:}\mathrm{mag}\to\mathrm{TOTAL}$):
\[
\Phi_{\mathrm{TOTAL}}\big((z_r)_{r\in S}\big)
\;=\;
\frac{1}{|S|}\sum_{r\in S}\Phi_{r{:}\mathrm{mag}}(z_r)
\;\in\;\mathbb{R}^{k}.
\]
This defines the TOTAL feature as a \emph{functorial average} consistent with the
poset morphisms:
\[
\Phi_{\mathrm{TOTAL}}\!\circ\!\iota_s
=\tfrac{1}{|S|}\,\Phi_{s{:}\mathrm{mag}},
\qquad
Y(v_s)
=\tfrac{1}{|S|}\,\mathrm{id}_{\mathbb{R}^{k}}.
\]
\item \textbf{Optional amplitude} (equivariant scalar summary):
\[
a_s \;\coloneqq\; \|x_s\|_2,
\]
which transforms under a gain rescaling $\lambda_s$ as 
$a_s \mapsto \lambda_s a_s$.
If a logarithmic scale is desired for analysis or classification,
$\log a_s$ can be applied downstream (outside the functor~$Y$).
\end{enumerate}

\paragraph{Transport maps on $Y$.}
For $(t,\lambda)\in M$, let $Y(t,\lambda)=\mathrm{id}$ on the spectral components
and define its action on the amplitude scalars by
\[
Y(t,\lambda)(a_s)=\lambda_s\,a_s.
\]
For the poset morphisms, set
\[
\begin{aligned}
Y(u_s:s{:}\mathrm{axes}\to s{:}\mathrm{mag})
   &= \mathrm{pr}_1:\mathbb{R}^k\times\mathbb{R}\to\mathbb{R}^k,\\
Y(v_s:s{:}\mathrm{mag}\to\mathrm{TOTAL})
   &= \tfrac{1}{|S|}\,\mathrm{id}_{\mathbb{R}^{k}}.
\end{aligned}
\]
These assignments make all naturality squares commute, yielding a well-defined
Group$\times$Poset-equivariant representation.

\subsection{Group invariance and poset naturality}

\begin{proposition}[Group invariance]
\label{prop:monoid}
We let $u_s: s{:}\mathrm{axes}\to s{:}\mathrm{mag}$ be the axis$\to$magnitude arrow in the poset,
and let $X(u_s)$ denote its action on signals. For any $(t,\lambda)\in C_T\times\Lambda$,
any $x_s\in X(s{:}\mathrm{axes})$, and any $x\in X(\mathrm{TOTAL})$,
\[
\Phi_{s{:}\mathrm{mag}}\!\big(X(u_s)((t,\lambda) \cdot x_s)\big)
=\Phi_{s{:}\mathrm{mag}}\!\big(X(u_s)(x_s)\big),
\quad
\Phi_{\mathrm{TOTAL}}\!\big((t,\lambda)\cdot x\big)
=\Phi_{\mathrm{TOTAL}}(x)
\]
and the amplitude scalar satisfies
\(
a_s\big((t,\lambda)\cdot x_s\big)=\lambda_s\,a_s(x_s).
\)
\end{proposition}

\begin{proof}
Write $z_s:=X(u_s)(x_s)\in X(s{:}\mathrm{mag})=\mathbb{R}^T$. Under the $M$-action,
\[
X(t,\lambda): z_s \mapsto z_s'=\lambda_s\,\tau_t z_s.
\]
Let $N(z)=z/\|z\|_2$ when $\|z\|_2>0$ and $N(0)=0$. By homogeneity and the
shift-invariance of $\|\cdot\|_2$, 
\[
N(z_s')=N(\lambda_s\,\tau_t z_s)=\tau_t N(z_s).
\]
The discrete real FFT satisfies the shift identity
$\widehat{\tau_t w}[r]=e^{-2\pi i rt/T}\widehat w[r]$, hence
$|\widehat{\tau_t w}[r]|=|\widehat w[r]|$ for all bins $r$. Therefore,
\begin{align*}
\Phi_{s{:}\mathrm{mag}}\!\big(X(u_s)((t,\lambda)\!\cdot\!x_s)\big)
&= \big|\mathrm{rFFT}(N(z_s'))\big|_{[1..k]}  \\
&= \big|\mathrm{rFFT}(\tau_t N(z_s))\big|_{[1..k]}  \\
&= \big|\mathrm{rFFT}(N(z_s))\big|_{[1..k]}  \\
&= \Phi_{s{:}\mathrm{mag}}\!\big(X(u_s)(x_s)\big).
\end{align*}
For the TOTAL node, $\Phi_{\mathrm{TOTAL}}((z_r)_r)=\frac1{|S|}\sum_{r\in S}\Phi_{r{:}\mathrm{mag}}(z_r)$,
and each summand is invariant to $(t,\lambda)$, so the average is invariant as well.
Finally, for the amplitude scalar $a_s(x_s)=\|x_s\|_2$,
\[
a_s\big((t,\lambda)\!\cdot\!x_s\big)=\|\lambda_s\,\tau_t x_s\|_2
=\lambda_s\,\|x_s\|_2=\lambda_s\,a_s(x_s).
\]
\end{proof}

\begin{proposition}[Poset naturality]
\label{lem:naturality}
For each $s\in S$, the following commutative squares hold:
\[
\begin{aligned}
Y(u_s)\circ \Phi_{s{:}\mathrm{axes}}
&=\Phi_{s{:}\mathrm{mag}}\circ X(u_s),
&&u_s:s{:}\mathrm{axes}\to s{:}\mathrm{mag},\\[3pt]
Y(v_s)\circ \Phi_{s{:}\mathrm{mag}}
&=\Phi_{\mathrm{TOTAL}}\circ X(v_s),
&&v_s:s{:}\mathrm{mag}\to\mathrm{TOTAL}.
\end{aligned}
\]
\end{proposition}

\begin{proof}
By construction of the representation maps,
\[
\Phi_{s{:}\mathrm{axes}}(x_s)
=\big(\,\Phi_{s{:}\mathrm{mag}}(\tilde m_s(x_s)),\,a_s(x_s)\,\big),
\qquad
\tilde m_s=X(u_s:s{:}\mathrm{axes}\to s{:}\mathrm{mag}),
\]
and $\tilde m_s$ satisfies
$\tilde m_s(x/\|x\|_2)=\tilde m_s(x)/\|\tilde m_s(x)\|_2$,
so normalization and pooling commute.
Hence
\[
Y(u_s)\circ \Phi_{s{:}\mathrm{axes}}
=\mathrm{pr}_1\!\big(\Phi_{s{:}\mathrm{mag}}\!\circ \tilde m_s,\,a_s\big)
=\Phi_{s{:}\mathrm{mag}}\!\circ \tilde m_s
=\Phi_{s{:}\mathrm{mag}}\!\circ X(u_s),
\]
which proves the first equality.

For the second, recall that
$X(v_s:s{:}\mathrm{mag}\to\mathrm{TOTAL})=\iota_s$
is the canonical inclusion into the $s$-th coordinate of
$X(\mathrm{TOTAL})=\prod_{r\in S}X(r{:}\mathrm{mag})$,
and that
\[
Y(v_s)=\tfrac{1}{|S|}\,\mathrm{id}_{\mathbb{R}^k},
\qquad
\Phi_{\mathrm{TOTAL}}\big((z_r)_{r\in S}\big)
=\tfrac{1}{|S|}\sum_{r\in S}\Phi_{r{:}\mathrm{mag}}(z_r).
\]
Therefore,
$\Phi_{\mathrm{TOTAL}}\circ \iota_s
=\tfrac{1}{|S|}\Phi_{s{:}\mathrm{mag}}$,
so both sides of the claimed equality reduce to
$\tfrac{1}{|S|}\Phi_{s{:}\mathrm{mag}}$.
Thus both naturality squares commute.
\end{proof}

The results above establish group invariance and poset naturality; together with the generator argument of Lemma~\ref{lem:generators} below, these imply the full category-equivariance theorem proved in Section~\ref{sec:theory}.

\subsection{Remarks for HAR}

In the experiments below, we use $S=\{\mathrm{ACC},\mathrm{GYRO}\}$, window length $T=128$, and low-frequency spectral bins $k=24$.
The resulting feature vector is
\[
\begin{aligned}
\Phi(x)
   &= [\,S_{\mathrm{ACC}}|_{[1..k]},\;
        S_{\mathrm{GYRO}}|_{[1..k]},\;
        \tfrac{1}{2}(S_{\mathrm{ACC}}{+}S_{\mathrm{GYRO}})|_{[1..k]},\\
   &\qquad\;\; a_{\mathrm{ACC}},\;
        a_{\mathrm{GYRO}}\,]
      \in \mathbb{R}^{3k+2}.
\end{aligned}
\]
Optionally, a pointwise $\log$ transform can be applied to $a_s$ downstream
(before classification) for numerical convenience.
This realization enforces \emph{both} group invariance (time and gain) and poset naturality (sensor hierarchy), yielding robustness under temporal shifts, amplitude drift, and axis rotations while maintaining interpretability and low model complexity.

\section{Experiments and Results}
\label{sec:experiments}

\subsection{Setup}

\paragraph{Dataset and split.}
We use the \emph{UCI Human Activity Recognition (HAR)} dataset with raw \emph{Inertial Signals} 
(tri-axial accelerometer and gyroscope; sampling window $T{=}128$).  
All experiments follow the official train/test split (six activity classes).  
No data augmentation is applied during training, as invariance is built into the representation $\Phi$.

\paragraph{Evaluation.}
We report \emph{classification accuracy}; we also include weighted F1 in the table.  
To assess robustness, we evaluate under a test-time distribution shift aligned with the group actions in our framework (time and gain) and with physically realistic orientation perturbations.

\paragraph{Test-time OOD (Out-of-Distribution) perturbations.}
For each test window we apply:
\begin{enumerate}
\item \textbf{Time shift (circular):} $\tau_{\Delta t}$ with $\Delta t\!\sim\!\mathrm{Unif}\{\!-18,\dots,18\!\}$.
\item \textbf{Per-sensor gain:} ACC and GYRO blocks scaled independently by $g\!\sim\!\mathrm{Unif}[0.7,1.4]$.
\item \textbf{Orientation:} independent random $R\!\in\!SO(3)$ (Haar-uniform) applied to each tri-axial block (device pose variation).
\end{enumerate}
Training data remain unperturbed; invariance is enforced analytically by construction.

\subsection{Models}

All models share the same simple classifier: 
StandardScaler $\rightarrow$ multinomial logistic regression 
(max\_iter $=1000$, $C{=}2.0$).  
Only the \emph{representation} $\Phi$ differs between models.

\paragraph{Baseline (raw time series).}
Concatenate the centered time series $(6{\times}128){=}\,768$ features; 
apply per-channel $z$-score normalization over time.

\paragraph{Group-only.}
Per-axis rFFT magnitudes (time-shift invariance) after per-sensor RMS normalization (gain invariance); no axis pooling (orientation-sensitive). Feature dimension $=6k$ with $k{=}24$ low-frequency bins.

\paragraph{Poset-only.}
Axis$\to$magnitude pooling 
$m(t)=\sqrt{x^2+y^2+z^2}$ (orientation invariant) 
followed by rFFT magnitudes; 
no RMS normalization (gain-sensitive).  
Feature dimension $=2k$.

\paragraph{Group$\times$Poset.}
%For each sensor: 
RMS normalization (gain invariance); 
axis$\to$magnitude pooling (poset);
rFFT magnitudes (time-shift invariance); retaining $k{=}24$ low-frequency bins.
The parent \textsc{Total} node is defined as the arithmetic mean of ACC and GYRO spectra.
Optionally, amplitude scalars $a_{\mathrm{ACC}}$ and $a_{\mathrm{GYRO}}$
(with an optional downstream $\log$ transform) are appended.
Total feature dimension $=3k+2=74$.

\subsection{Main OOD results}

\paragraph{Single OOD draw.}
Table~\ref{tab:single} compares accuracy and weighted F1 for one random OOD (Out-of-Distribution) realization.

\begin{table}[H]
\centering
\begin{tabular}{lccc}
\toprule
Model & Accuracy & Weighted-F1 & Dim. \\
\midrule
Baseline (raw) & 0.1880 & 0.1855 & 768 \\
Group-only (rFFT per-axis, RMS) & 0.4442 & 0.4340 & 144 \\
Poset-only (axis$\to$mag, no RMS) & 0.5979 & 0.5972 & 48 \\
\textbf{Group$\times$Poset (ours)} & \textbf{0.6447} & \textbf{0.6382} & 74 \\
\bottomrule
\end{tabular}
\vspace{5pt}
\caption{Out-of-distribution performance on UCI HAR under circular time shift ($\pm18$ samples), per-sensor gain $U[0.7,1.4]$, and random $SO(3)$ axis rotation (Haar-uniform).}
\label{tab:single}
\end{table}

\paragraph{Robustness across OOD draws.}
We repeated OOD sampling with five independent seeds (perturbations resampled each time).  
Table~\ref{tab:robust} reports the mean $\pm$ standard deviation of accuracy.

\begin{table}[H]
\centering
\begin{tabular}{lcc}
\toprule
Model & Accuracy (mean) & Accuracy (std) \\
\midrule
Baseline (raw) & 0.1773 & 0.0086 \\
Group-only & 0.4408 & 0.0057 \\
Poset-only & 0.5901 & 0.0043 \\
\textbf{Group$\times$Poset (ours)} & \textbf{0.6415} & \textbf{0.0041} \\
\bottomrule
\end{tabular}
\vspace{5pt}
\caption{Robustness under five independent OOD realizations on UCI HAR. 
The Group$\times$Poset model achieves $\approx3.6\times$ the baseline accuracy.}
\label{tab:robust}
\end{table}

\subsection{Ablation analysis and interpretation}

\paragraph{Decomposing the gains.}
Relative to the baseline, the Poset-only model (orientation invariance) recovers $+0.41$ absolute accuracy under OOD perturbations, 
while the Group-only model (time/gain invariance) recovers $+0.26$.  
Combining both yields the strongest robustness: the full Group$\times$Poset category-equivariant representation 
improves a further $+0.05$ over Poset-only and $+0.20$ over Group-only, reflecting \emph{complementary} protections.

\paragraph{Dimension efficiency.}
Despite using approximately $10\times$ fewer features than the raw baseline (74 vs.\ 768), the Group$\times$Poset model achieves the highest OOD accuracy, underscoring that \emph{structure} outperforms capacity when \emph{label-invariant} factors align with known symmetries.

\paragraph{Why Baseline fails.}
Under OOD, the raw coordinate system changes (rotation), the energy scale drifts (gain), and the phase shifts (time). Without equivariance, the representation moves along large orbits in feature space; the linear head cannot compensate. Our construction collapses these orbits by design, preserving class information while discarding label-invariant variation. 

\subsection{Summary}
Under realistic, test-only perturbations (time, gain, pose), our \emph{Group$\times$Poset} representation achieves approximately $3.6\times$ the baseline accuracy and remains stable across OOD draws, validating the categorical design with both group invariance and poset naturality. It should be noted that we observed similar trends for $k=16,32$. Note also that \emph{all} models share identical training and differ only in the equivariant feature map.

\section{Theoretical Analysis}
\label{sec:theory}

Throughout, let $\mathcal{C}=B M\times P$ denote the product category defined in Section~\ref{sec:framework}, with 
data and model functors $X,Y:\mathcal{C}\to\mathbf{Meas}$ 
and representation $\Phi=\{\Phi_a:X(a)\to Y(a)\}_{a\in\mathrm{Ob}(P)}$.

We first verify that the group action is natural with respect to the $P$–arrows:

\begin{lemma}[M-action natural in $P$]\label{lem:M-action-natural}
Suppose for each object $a\in \mathrm{Ob}(P)$ we are given group actions
\[
A^X_a:\; M\to \mathrm{End}\bigl(X_P(a)\bigr),\qquad
A^Y_a:\; M\to \mathrm{End}\bigl(Y_P(a)\bigr),
\]
i.e.\ $A^\bullet_a(e)=\mathrm{id}$ and $A^\bullet_a(m'm)=A^\bullet_a(m')\circ A^\bullet_a(m)$ for all $m,m'\in M$ where $\bullet=X,Y$.
Then the following are equivalent:
\begin{enumerate}\itemsep4pt
\item There exist functors $X,Y:\mathbf{B}M\times P\to \mathbf{Meas}$ extending $X_P,Y_P$ on objects and $P$-arrows such that, for every morphism $(m,u):a\to b$,
\begin{equation}\label{eq:product-rule}
X(m,u)=X_P(u)\circ A^X_a(m),\qquad
Y(m,u)=Y_P(u)\circ A^Y_a(m).
\end{equation}
\item For each $m\in M$, the families $\{A^X_a(m)\}_{a\in \mathrm{Ob}(P)}$ and $\{A^Y_a(m)\}_{a\in \mathrm{Ob}(P)}$ are natural endomorphisms of $X_P$ and $Y_P$, i.e.
\begin{equation}\label{eq:nat-square}
X_P(u)\circ A^X_a(m)=A^X_b(m)\circ X_P(u),\qquad
Y_P(u)\circ A^Y_a(m)=A^Y_b(m)\circ Y_P(u),
\end{equation}
for all $u:a\to b$ in $P$.
\end{enumerate}
Moreover, when (2) holds the assignment \eqref{eq:product-rule} defines unique functors $X,Y$ on $C$.
\end{lemma}

\begin{proof}
(1 $\Rightarrow$ 2): In $\mathbf{B}M\times P$ we have $(m,\mathrm{id}_b)\circ (e,u)=(m,u)=(e,u)\circ (m,\mathrm{id}_a)$.
Applying $X$ and writing $A^X_a(m):=X(m,\mathrm{id}_a)$ gives
\[
X_P(u)\circ A^X_a(m)
= X(e,u)\circ X(m,\mathrm{id}_a)
= X(m,\mathrm{id}_b)\circ X(e,u)
= A^X_b(m)\circ X_P(u),
\]
and similarly for $Y$, yielding \eqref{eq:nat-square}.

(2 $\Rightarrow$ 1): Define $X$ on morphisms by \eqref{eq:product-rule} and check functoriality.
Identities: $X(e,\mathrm{id}_a)=X_P(\mathrm{id}_a)\circ A^X_a(e)=\mathrm{id}$.
Composition: for $(m,u):a\to b$ and $(m',u'):b\to c$,
\[
\begin{aligned}
X(m',u')\circ X(m,u)
&=\bigl(X_P(u')\circ A^X_b(m')\bigr)\circ \bigl(X_P(u)\circ A^X_a(m)\bigr)\\
&= X_P(u')\circ \underbrace{A^X_b(m')\circ X_P(u)}_{\text{by \eqref{eq:nat-square}}=\,X_P(u)\circ A^X_a(m')}\circ A^X_a(m)\\
&= X_P(u'u)\circ \bigl(A^X_a(m')\circ A^X_a(m)\bigr)
= X_P(u'u)\circ A^X_a(m'm)\\
&= X(m'm,u'u).
\end{aligned}
\]
Uniqueness is immediate from \eqref{eq:product-rule}. The same argument holds for $Y$. This completes the proof.
\end{proof}

This can be applied in the case of HAR data and model functors. 
Fix $T\in\mathbb{N}$ and a sensor set $S$. For $s\in S$ let
\[
X_P(s\!:\!\mathrm{axes})=(\mathbb{R}^T)^3,\quad
X_P(s\!:\!\mathrm{mag})=\mathbb{R}^T,\quad
X_P(\mathrm{TOTAL})=\prod_{r\in S}\mathbb{R}^T.
\]
The non-identity arrows of $P$ are $u_s:s\!:\!\mathrm{axes}\to s\!:\!\mathrm{mag}$ and $v_s:s\!:\!\mathrm{mag}\to \mathrm{TOTAL}$ with
\[
X_P(u_s)=m_s:(x^{(1)},x^{(2)},x^{(3)})\mapsto \bigl(n\mapsto \sqrt{\sum_{i=1}^3 (x^{(i)}(n))^2}\bigr),\qquad
X_P(v_s)=\iota_s,
\]
and $M=C_T\times \Lambda$ acts by $A^X_{s:\mathrm{axes}}(t,\lambda)=\lambda_s\,\tau_t$ componentwise, $A^X_{s:\mathrm{mag}}(t,\lambda)=\lambda_s\,\tau_t$, and $A^X_{\mathrm{TOTAL}}(t,\lambda)$ acting componentwise on the product.
Naturality for $u_s$ can be seen as follows: For $x\in (\mathbb{R}^T)^3$ and $n\in\{0,\dots,T-1\}$,
\begin{align*}
\bigl(X_P(u_s)\circ A^X_{s:\mathrm{axes}}(t,\lambda)\bigr)(x)(n)
 &= m_s\bigl(\lambda_s\,\tau_t x\bigr)(n)\\
 &= \lambda_s\sqrt{\textstyle\sum_{i=1}^3 (x^{(i)}(n+t))^2}\\
 &= \bigl(A^X_{s:\mathrm{mag}}(t,\lambda)\circ X_P(u_s)\bigr)(x)(n).
\end{align*}
Naturality for $v_s$ can be seen as follows: For $z\in\mathbb{R}^T$,
\[
\bigl(X_P(v_s)\circ A^X_{s:\mathrm{mag}}(t,\lambda)\bigr)(z)
= \iota_s\bigl(\lambda_s\,\tau_t z\bigr)
= \bigl(A^X_{\mathrm{TOTAL}}(t,\lambda)\circ X_P(v_s)\bigr)(z).
\]
For the model functor $Y_P$, take $Y_P(s\!:\!\mathrm{axes})=\mathbb{R}^k\times \mathbb{R}$, $Y_P(s\!:\!\mathrm{mag})=\mathbb{R}^k$, $Y_P(\mathrm{TOTAL})=\mathbb{R}^k$, with
\[
Y_P(u_s)=\mathrm{pr}_1,\qquad
Y_P(v_s)=\tfrac{1}{|S|}\,\mathrm{id}_{\mathbb{R}^k},
\]
and let $A^Y_{s:\mathrm{axes}}(t,\lambda):(v,a)\mapsto (v,\lambda_s a)$ while $A^Y_{s:\mathrm{mag}}(t,\lambda)=A^Y_{\mathrm{TOTAL}}(t,\lambda)=\mathrm{id}$.
Then we have 
\begin{align*}
Y_P(u_s)\circ A^Y_{s:\mathrm{axes}}(t,\lambda)
  &= \mathrm{pr}_1
   = A^Y_{s:\mathrm{mag}}(t,\lambda)\circ Y_P(u_s),\\
Y_P(v_s)\circ A^Y_{s:\mathrm{mag}}(t,\lambda)
  &= A^Y_{\mathrm{TOTAL}}(t,\lambda)\circ Y_P(v_s).
\end{align*}
so \eqref{eq:nat-square} holds for $X_P$ and $Y_P$ on the generators of $P$. Since $P$ is thin, naturality holds for all $P$-arrows.

In the following part of the section, we formalize the categorical guarantees of the proposed representation and quantify its robustness under the transformations captured by the group. 

\subsection{Verifying naturality on generators suffices}

Let $M=C_T\times \Lambda$, where $C_T=\langle \tau^1 \rangle$ is the cyclic group generated by the one-step circular shift, 
and $\Lambda=(\mathbb{R}_{>0})^S$ is the commutative group of per-sensor gains generated by coordinatewise scalings 
$\{\lambda^{(s)}\}_{s\in S}$ defined by $(\lambda^{(s)})_s=\lambda>0$ and $(\lambda^{(s)})_{s'\neq s}=1$.  
Let $P$ be a thin category (poset) whose non-identity morphisms correspond to the inclusions 
\emph{axis$\to$mag} and \emph{mag$\to$TOTAL}.

\begin{lemma}[Generators determine naturality]
\label{lem:generators}
If the naturality condition
\[
Y(m,u)\circ \Phi_a \;=\; \Phi_b\circ X(m,u)
\qquad\text{for }(m,u):a\to b,
\]
holds for the generating morphisms
\[
(\tau^1,\mathrm{id}),\qquad 
(\lambda^{(s)},\mathrm{id}),\qquad 
(\mathrm{id},\,u_{s:\mathrm{axes}\to s:\mathrm{mag}}),\qquad
(\mathrm{id},\,v_{s:\mathrm{mag}\to \mathrm{TOTAL}}),
\]
then it holds for all morphisms of $\mathcal{C}$.
\end{lemma}

\begin{proof}
Every morphism $(m,u)\in \mathrm{Mor}(\mathcal{C})$ can be expressed as a finite composite of generators, 
since $M$ is generated by $\tau^1$ and $\{\lambda^{(s)}\}_{s\in S}$, 
and $P$ is thin, with morphisms generated by $\{u_s,v_s\}_{s\in S}$.  
By functoriality of $X$ and $Y$, and repeated use of the assumed commutative squares on the generators, 
the composite square for $(m,u)$ also commutes.
\end{proof}

\subsection{Equivariance and invariance of the concrete operators}

\begin{proposition}[Invariances of $\Phi$'s primitives]
\label{prop:invars}
For any tri-axial window $x=(x^{(1)},x^{(2)},x^{(3)})\in(\mathbb{R}^T)^3$:
\begin{enumerate}
\item[\textup{(a)}] \textbf{Gain invariance (RMS).}
For $\lambda>0$,
\[
\widehat{\lambda x}
=\frac{\lambda x}{\|\lambda x\|_2}
=\frac{x}{\|x\|_2}
=\widehat{x}.
\]
Moreover, the amplitude satisfies
\[
a(\lambda x)=\lambda\,a(x),
\qquad a(x)=\|x\|_2.
\]

\item[\textup{(b)}] \textbf{Orientation invariance (axis$\to$mag).}
For any $R\in SO(3)$ and
$m(n)=\sqrt{\sum_i (\widehat{x}^{(i)}(n))^2}$,
\[
m_R(n)
:=\sqrt{\|R\,\widehat{x}(n)\|_2^2}
=\sqrt{\widehat{x}(n)^\top R^\top R\,\widehat{x}(n)}
=m(n).
\]

\item[\textup{(c)}] \textbf{Time-shift invariance (rFFT magnitude).}
For a magnitude signal $m\in\mathbb{R}^T$ and a circular shift $\tau_t$,
the discrete Fourier transform satisfies
\[
\widehat{\tau_t m}[r]
= e^{-2\pi i r t / T}\,\widehat{m}[r]
\quad\text{for all frequency bins }r,
\]
so each bin is multiplied by a unit-modulus phase factor.
Taking absolute values removes this phase, yielding
\[
|\mathrm{rFFT}(\tau_t m)[r]| = |\mathrm{rFFT}(m)[r]| \quad (\forall\,r).
\]
Hence the truncated magnitude
$|\mathrm{rFFT}(m)|_{[1..k]}$ is invariant to circular shifts.
\end{enumerate}
\end{proposition}

\begin{proof}
(a) follows immediately from the homogeneity of the $\ell_2$ norm.  
(b) uses $R^\top R = \mathrm{Id}_3$.  
(c) For the length-$T$ real FFT, the discrete shift property gives
\[
\widehat{\tau_t m}[k]
= e^{-2\pi i k t / T}\,\widehat{m}[k],
\]
so each frequency bin is multiplied by a unit-modulus phase factor.
Taking absolute values removes this phase, yielding
\[
|\widehat{\tau_t m}[k]| = |\widehat{m}[k]|
\quad\text{for all }k.
\]
\end{proof}

\begin{lemma}[Poset naturality of pooling]
\label{lem:poset}
Let $\tilde m_s:(\mathbb{R}^T)^3\to\mathbb{R}^T$ be the axis$\to$mag map and
$v_s:s{:}\mathrm{mag}\to\mathrm{TOTAL}$ the inclusion arrow.
With
\begin{align*}
\Phi_{s{:}\mathrm{mag}}(z)
  &= |\mathrm{rFFT}(N(z))|_{[1..k]},\\[3pt]
\Phi_{s{:}\mathrm{axes}}(x_s)
  &= \bigl(\,\Phi_{s{:}\mathrm{mag}}(\tilde m_s(x_s)),\,a_s(x_s)\,\bigr),
\end{align*}
and
\begin{align*}
Y(u_s) &= \mathrm{pr}_1 : \mathbb{R}^k\times\mathbb{R} \to \mathbb{R}^k,\\[3pt]
X(\mathrm{TOTAL}) &= \prod_{r\in S}\mathbb{R}^T,\qquad
X(v_s) = \iota_s,\qquad
Y(v_s) = \tfrac{1}{|S|}\,\mathrm{id}_{\mathbb{R}^k},
\end{align*}
the naturality squares for
\(
u_s : s{:}\mathrm{axes}\to s{:}\mathrm{mag}
\)
and
\(
v_s : s{:}\mathrm{mag}\to \mathrm{TOTAL}
\)
commute.
\end{lemma}

\begin{proof}
By construction,
\[
\Phi_{s{:}\mathrm{axes}}(x_s)
=\big(\,\Phi_{s{:}\mathrm{mag}}(\tilde m_s(x_s)),\,a_s(x_s)\,\big),
\qquad
\tilde m_s=X(u_s:s{:}\mathrm{axes}\to s{:}\mathrm{mag}).
\]
Hence
\[
Y(u_s)\!\circ\!\Phi_{s{:}\mathrm{axes}}
=\mathrm{pr}_1\!\big(\Phi_{s{:}\mathrm{mag}}\!\circ\!\tilde m_s,\,a_s\big)
=\Phi_{s{:}\mathrm{mag}}\!\circ\!\tilde m_s
=\Phi_{s{:}\mathrm{mag}}\!\circ\!X(u_s),
\]
which proves the first square.
For $v_s$, note that $Y(v_s)=\tfrac{1}{|S|}\,\mathrm{id}_{\mathbb{R}^k}$ and
\(
\Phi_{\mathrm{TOTAL}}\!\circ\!\iota_s=\tfrac{1}{|S|}\,\Phi_{s{:}\mathrm{mag}},
\)
so both sides coincide and the second square commutes.
\end{proof}

\begin{theorem}[Category-equivariance on $\mathcal{C}$]
\label{thm:catequiv-full}
The representation $\Phi$ is a natural transformation 
$X\Rightarrow Y$ 
on the product category $\mathcal{C}=B M\times P$.
\end{theorem}

\begin{proof}
By Proposition~\ref{prop:invars}, 
naturality holds for 
$(\tau^1,\mathrm{id})$ and $(\lambda^{(s)},\mathrm{id})$ 
(time and gain).  
By Lemma~\ref{lem:poset}, 
it holds for 
$(\mathrm{id},u_s)$ and $(\mathrm{id},v_s)$.  
The claim then follows from Lemma~\ref{lem:generators}.
\end{proof}

\subsection{Robust risk under group perturbations}

Let $\mathcal{D}$ be a distribution on labeled windows $(x,y)$ and let $\mu$ be a distribution on $M$ (test-time perturbations).
Define the perturbed distribution $\mathcal{D}_\mu$ by sampling $(x,y)\!\sim\!\mathcal{D}$ and $m\!\sim\!\mu$, then emitting $(m\!\cdot\!x,y)$.
Let $h:\mathbb{R}^d\to\mathbb{R}^C$ be a predictor, and let $\ell:\mathbb{R}^C\times\{1,\dots,C\}\to\mathbb{R}_+$ be a loss that is $L_\ell$-Lipschitz in its first argument (w.r.t.\ $\|\cdot\|_2$).

\begin{proposition}[Exact robustness for $M$-invariant representations]
\label{prop:exact}
If a representation $I$ is $M$-invariant, i.e.\ $I(m\!\cdot\!x)=I(x)$ for all $m\in M$, then for any $h$ and $\ell$,
\[
\mathbb{E}_{(x,y)\sim\mathcal{D}_\mu}\!\big[\ell(h(I(x)),y)\big]
\;=\;
\mathbb{E}_{(x,y)\sim\mathcal{D}}\!\big[\ell(h(I(x)),y)\big].
\]
\end{proposition}

\begin{proof}
By definition of $\mathcal{D}_\mu$,
\[
\mathbb{E}_{(x,y)\sim\mathcal{D}_\mu}\!\big[\ell(h(I(x)),y)\big]
=\mathbb{E}_{(x,y)\sim\mathcal{D},\,m\sim\mu}\!\big[\ell(h(I(m\!\cdot\!x)),y)\big].
\]
Using $I(m\!\cdot\!x)=I(x)$ pointwise, the right-hand side equals
$\mathbb{E}_{(x,y)\sim\mathcal{D}}\![\ell(h(I(x)),y)]$.
(If desired, Tonelli applies directly since $\ell\ge 0$.)
\end{proof}

\begin{theorem}[Stability bound for non-equivariant representations]
\label{thm:stability}
Let $\psi$ be any representation and assume $h$ is $L$-Lipschitz (w.r.t.\ $\|\cdot\|_2$).
Assume moreover that $\mathbb{E}_{(x,y)\sim\mathcal{D},\,m\sim\mu}\|\psi(m\!\cdot\!x)-\psi(x)\|_2<\infty$.
Then
\begin{multline}
\Big|
  \mathbb{E}_{(x,y)\sim\mathcal{D}_\mu}\!\big[\ell(h(\psi(x)),y)\big]
 -\mathbb{E}_{(x,y)\sim\mathcal{D}}\!\big[\ell(h(\psi(x)),y)\big]
\Big|\\[2pt]
\le\;
L_\ell L\,
\mathbb{E}_{\substack{(x,y)\sim\mathcal{D}\\ m\sim\mu}}
 \!\big[\|\psi(m\!\cdot\!x)-\psi(x)\|_2\big].
\end{multline}
In particular, the excess risk is bounded by the expected \emph{$\mu$-orbit displacement} of $\psi$.
\end{theorem}

\begin{proof}
By the definition of $\mathcal D_\mu$,
\[
\mathbb{E}_{(x,y)\sim\mathcal{D}_\mu}\!\big[\ell(h(\psi(x)),y)\big]
=\mathbb{E}_{(x,y)\sim\mathcal D,\,m\sim\mu}\!\big[\ell(h(\psi(m\!\cdot\!x)),y)\big].
\]
Hence the left-hand side equals
\begin{multline}
\Big|
\mathbb{E}_{(x,y)\sim\mathcal D,\,m\sim\mu}
\big[\ell(h(\psi(m\!\cdot\!x)),y)-\ell(h(\psi(x)),y)\big]
\Big|\\[2pt]
\le\;
\mathbb{E}_{(x,y)\sim\mathcal D,\,m\sim\mu}
\Big|\ell(h(\psi(m\!\cdot\!x)),y)-\ell(h(\psi(x)),y)\Big|.
\end{multline}
where we used Jensen's inequality for $z\mapsto|z|$.
By the $L_\ell$-Lipschitz property of $\ell$ (first argument) and the $L$-Lipschitz property of $h$,
\[
\begin{aligned}
\Big|\ell\!\big(h(\psi(m\!\cdot\!x)),y\big)
      -\ell\!\big(h(\psi(x)),y\big)\Big|
&\;\le\;
L_\ell\,\big\|h(\psi(m\!\cdot\!x))-h(\psi(x))\big\|_2\\[2pt]
&\;\le\;
L_\ell L\,\|\psi(m\!\cdot\!x)-\psi(x)\|_2.
\end{aligned}
\]
Taking expectations yields the stated bound.
\end{proof}

Thus, if $\Phi$ is $M$-invariant on its spectral blocks and $P$-natural (Theorem~\ref{thm:catequiv-full}), then, under test-time $M$-perturbations $m\!\sim\!\mu$, the spectral component is exactly robust (Proposition~\ref{prop:exact}), while appended $M$-equivariant scalars (e.g., amplitude summaries) scale multiplicatively under gain.

%\begin{corollary}[Benefit of Group$\times$Poset structure]
%\label{cor:benefit}
%If $\Phi$ is $M$-invariant on its spectral blocks and $P$-natural (cf. Theorem~\ref{thm:catequiv-full}), then under test-time nuisances drawn from $\mu$ the spectral portion enjoys the exact robustness of Proposition~\ref{prop:exact}, while appended equivariant scalars (e.g.\ amplitude summaries) transform multiplicatively (hence predictably) under gain.
%\end{corollary}

Time-shift invariance obtained through spectral magnitudes necessarily discards phase information.  
Let $\Pi:\mathbb{C}^{T/2+1}\to\mathbb{R}^{T/2+1}$ denote the modulus map.  
For any two signals $m,m'$ satisfying $\Pi(\widehat{m})=\Pi(\widehat{m}')$ on the retained frequency bins, 
$\Phi$ cannot distinguish $m$ and $m'$ through its spectral blocks alone; 
discrimination must arise from other components (e.g.\ amplitude summaries or additional sensors) 
or from increasing $k$. This trade-off is explicit and controlled by the choice of $k$.

Because $P$ is thin, there exists at most one morphism $u:a\to b$, and any two composites $a\to b$ are equal.  
Hence, once the naturality squares commute on generators, they commute globally (Lemma~\ref{lem:generators}), 
ensuring that hierarchical pooling cannot introduce \emph{aliasing} across paths.

The TOTAL-node feature can be interpreted as a \emph{functorial average}
of its child nodes.
With $X(\mathrm{TOTAL})$ a product, $Y(v_s)=\tfrac{1}{|S|}\,\mathrm{id}$,
and $\Phi_{\mathrm{TOTAL}}$ defined as the arithmetic mean of the child
features, $P$-naturality guarantees that
\emph{transport then pool} equals \emph{pool then transport},
mirroring the consistency conditions familiar from sheaf theory on finite posets.

\section{Conclusion}
\label{sec:conclusion}

We presented a Group×Poset category-equivariant framework that models both compositional and hierarchical symmetries. Applied to HAR, the construction demonstrates large robustness gains and theoretical clarity via natural transformations. Future work includes extensions to spatio-temporal graphs and multimodal sensor networks.

Broadly speaking, this work advances category equivariance as a practical design principle for learning under structured variability. By modeling temporal and amplitude effects as a group and sensor hierarchy as a poset, then requiring representations to commute with both, we obtain models that are robust to realistic shifts (timing, gain, orientation) while remaining compact and transparent. Our Human Activity Recognition study demonstrates that such structure can be built directly into the representation, yielding substantial improvements under out-of-distribution conditions and making the source of robustness easy to audit.

The significance of category equivariance extends well beyond HAR. Many sensing and decision systems combine compositional transformations with hierarchical organization: multimodal wearables and IoT platforms (calibration drift, resampling, topology of devices), robotics and autonomous systems (viewpoint and scale with kinematic or task hierarchies), neuro/physiological monitoring (per-electrode gains with cortical region hierarchies), geospatial pipelines (map projections and resolutions with tiling hierarchies), and even knowledge-rich applications (taxonomy- or ontology-structured outputs). Categories naturally capture this mix of \emph{actions} (often non-invertible) and \emph{relations} (partial orders), providing a unifying abstraction where invariance and consistency are enforced by construction rather than by hope or augmentation.

Practically, the approach offers three recurring benefits. First, \textit{data efficiency}: parameter sharing across categorical morphisms reduces sample complexity and alleviates subject- or hardware-specific overfitting. Second, \textit{robustness}: group-invariant blocks eliminate sensitivity along label-invariant directions and stabilize performance under deployment shifts. Third, \textit{interpretability}: commuting diagrams make assumptions explicit, so one can trace which symmetries are protected and where discriminative information lives. Category equivariance is also complementary: it blends smoothly with group-equivariant layers, sheaf- or graph-based models, and modern sequence architectures, adding hierarchical consistency and non-invertible symmetries to the existing toolbox.

Several avenues follow naturally. On the modeling side, learnable layers with \emph{built-in naturality}—imposing the commuting constraints during training—can replace fixed operators while preserving guarantees. Richer symmetry classes (e.g., time-warp semigroups, photometric or sensor-specific semigroups) and also deeper hierarchies (multi-scale spatial partitions, taxonomies) should broaden applicability. On the evaluation side, standardized robustness suites that isolate group and poset components would make comparisons more meaningful than aggregate accuracy alone. Finally, certifiable bounds—combining invariance risk guarantees with category-theoretic consistency—could translate structural assumptions into actionable deployment assurances.

In short, category equivariance elevates robustness from an afterthought to a \emph{first-class architectural constraint}. By aligning models with the algebra of how data are generated and organized, it turns abstract structure into concrete reliability across domains where change is the rule rather than the exception.


\begin{thebibliography}{99}

\bibitem{Anguita2013}
D.~Anguita, A.~Ghio, L.~Oneto, X.~Parra, and J.~L.~Reyes-Ortiz.
A Public Domain Dataset for Human Activity Recognition using Smartphones.
In \emph{Proc. 21st European Symposium on Artificial Neural Networks, Computational Intelligence and Machine Learning (ESANN 2013)}, pp.~437--442, 2013.

\bibitem{Arjovsky2019}
M.~Arjovsky, L.~Bottou, I.~Gulrajani, and D.~Lopez-Paz.
Invariant Risk Minimization.
arXiv:1907.02893, 2019.

\bibitem{Bronstein2021}
M.~M.~Bronstein, J.~Bruna, T.~Cohen, and P.~Veli\v{c}kovi\'{c}.
Geometric Deep Learning: Grids, Groups, Graphs, Geodesics, and Gauges.
arXiv:2104.13478, 2021.

\bibitem{CohenWelling2016}
T.~S.~Cohen and M.~Welling.
Group Equivariant Convolutional Networks.
In \emph{Proc. 33rd International Conference on Machine Learning (ICML)}, PMLR~48, pp.~2990--2999, 2016.

\bibitem{CohenWelling2017}
T.~S.~Cohen and M.~Welling.
Steerable CNNs.
In \emph{International Conference on Learning Representations (ICLR)}, 2017.

\bibitem{Cohen2019Gauge}
T.~Cohen, M.~Weiler, B.~Kicanaoglu, and M.~Welling.
Gauge Equivariant Convolutional Networks and the Icosahedral CNN.
In \emph{Proc. 36th International Conference on Machine Learning (ICML)}, PMLR~97, pp.~1321--1330, 2019.

\bibitem{Finzi2021EMLP}
M.~Finzi, M.~Welling, and A.~G.~Wilson.
A Practical Method for Constructing Equivariant Multilayer Perceptrons for Arbitrary Matrix Groups.
In \emph{Proc. 38th International Conference on Machine Learning (ICML)}, PMLR~139, pp.~3318--3328, 2021.

\bibitem{FongSpivak2019}
B.~Fong and D.~I.~Spivak.
\emph{Seven Sketches in Compositionality: An Invitation to Applied Category Theory}.
Cambridge University Press, 2019.

\bibitem{HansenGhrist2019}
J.~Hansen and R.~Ghrist.
Toward a Spectral Theory of Cellular Sheaves.
\emph{Journal of Applied and Computational Topology}, 3(4):315--358, 2019.

\bibitem{KerivenPeyre2019}
N.~Keriven and G.~Peyr\'{e}.
Universal Invariant and Equivariant Graph Neural Networks.
In \emph{Advances in Neural Information Processing Systems (NeurIPS)}, 2019.

\bibitem{KondorTrivedi2018}
R.~Kondor and S.~Trivedi.
On the Generalization of Equivariance and Convolution in Neural Networks to the Action of Compact Groups.
In \emph{Proc. 35th International Conference on Machine Learning (ICML)}, PMLR~80, pp.~2747--2755, 2018.

\bibitem{Kwapisz2011}
J.~R.~Kwapisz, G.~M.~Weiss, and S.~A.~Moore.
Activity Recognition Using Cell Phone Accelerometers.
\emph{ACM SIGKDD Explorations Newsletter}, 12(2):74--82, 2011.

\bibitem{MacLane1998}
S.~Mac~Lane.
\emph{Categories for the Working Mathematician}, 2nd~ed.
Graduate Texts in Mathematics~5, Springer, 1998.

\bibitem{Maron2019}
H.~Maron, H.~Ben-Hamu, N.~Shamir, and Y.~Lipman.
Invariant and Equivariant Graph Networks.
In \emph{International Conference on Learning Representations (ICLR)}, 2019.

\bibitem{OppenheimSchafer2009}
A.~V.~Oppenheim and R.~W.~Schafer.
\emph{Discrete-Time Signal Processing}, 3rd~ed.
Pearson/Prentice Hall, 2009.

\bibitem{SakoeChiba1978}
H.~Sakoe and S.~Chiba.
Dynamic Programming Algorithm Optimization for Spoken Word Recognition.
\emph{IEEE Transactions on Acoustics, Speech, and Signal Processing}, 26(1):43--49, 1978.

\bibitem{Spivak2014}
D.~I.~Spivak.
\emph{Category Theory for the Sciences}.
MIT Press, 2014.

\bibitem{Thomas2018}
N.~Thomas, T.~Smidt, S.~Kearnes, L.~Yang, L.~Li, K.~Kohlhoff, and P.~Riley.
Tensor Field Networks: Rotation- and Translation-Equivariant Neural Networks for 3D Point Clouds.
arXiv:1802.08219, 2018.

\bibitem{WeilerCesa2019}
M.~Weiler and G.~Cesa.
General \(E(2)\)-Equivariant Steerable CNNs.
In \emph{Advances in Neural Information Processing Systems (NeurIPS)}, 2019.

\bibitem{Zaheer2017}
M.~Zaheer, S.~Kottur, S.~Ravanbakhsh, B.~P\'{o}czos, R.~R.~Salakhutdinov, and A.~J.~Smola.
Deep Sets.
In \emph{Advances in Neural Information Processing Systems 30}, pp.~3391--3401, 2017.

\end{thebibliography}
\end{document}